\pdfoutput=1
\documentclass{article}
\usepackage{iclr2027_conference,times}
\iclrpreprint
\usepackage{amsmath,amssymb,amsthm,mathtools,bm}
\usepackage{booktabs,multirow,array}
\usepackage{colortbl}
\usepackage{flafter}
\usepackage{graphicx,subcaption}
\usepackage{algorithm}
\usepackage[noend]{algpseudocode}
\usepackage{microtype}
\usepackage{enumitem}
\usepackage{hyperref}
\usepackage{url}
\hypersetup{colorlinks=true,linkcolor=[rgb]{0.12,0.23,0.50},citecolor=[rgb]{0.12,0.23,0.50},urlcolor=[rgb]{0.12,0.23,0.50}}
\hypersetup{pdftitle={Refresh or Realize? Compute Allocation in Drifting Models},
  pdfauthor={Sipeng Chen, Xu Zheng, Shibo Li}}

\usepackage[nameinlink,capitalize]{cleveref}
\usepackage{xspace}
\newcommand{\R}{\mathbb{R}}
\newcommand{\E}{\mathbb{E}}
\newcommand{\sg}{\operatorname{sg}}
\newcommand{\Range}{\operatorname{Range}}
\newcommand{\Null}{\operatorname{Null}}
\newcommand{\diag}{\operatorname{diag}}

\newcommand{\norm}[1]{\left\lVert #1\right\rVert}
\newcommand{\ip}[2]{\left\langle #1,#2\right\rangle}
\newcommand{\pinv}{\dagger}
\newcommand{\kmat}{\mathbf{K}}
\newcommand{\jmat}{\mathbf{J}}
\newcommand{\vstack}{\mathbf{V}}
\newcommand{\hstack}{\mathbf{H}}
\newcommand{\thet}{\boldsymbol{\theta}}

\newtheorem{theorem}{Theorem}
\newtheorem{proposition}[theorem]{Proposition}

\newtheorem{lemma}[theorem]{Lemma}

\title{Refresh or Realize? Compute Allocation in Drifting Models}
\author{%
Sipeng Chen \\
Department of Computer Science\\
Florida State University\\
\texttt{sc25bg@fsu.edu}
\And
Xu Zheng \\
Knight Foundation School of Computing\\
and Information Sciences\\
Florida International University\\
\texttt{xzhen019@fiu.edu}
\And
Shibo Li \\
Department of Computer Science\\
Florida State University\\
\texttt{sl24bp@fsu.edu}
}

\begin{document}
\maketitle

\begin{abstract}
Drifting Models train a one-step generator by recomputing a finite-sample drift field at every
iteration and taking an optimizer step toward the drifted target. The field says how generated
samples should move, but the step is taken in parameters shared by all samples, so the motion the
network actually makes need not match the motion it was given. This leaves a basic training question
open: should extra compute go into fitting the current target more closely, or into recomputing the
field? We study it on ImageNet $256\times256$. Holding the target fixed for $k$ optimizer steps and
measuring the realized displacement, we find that deeper fitting does bring the network closer to
the frozen target, and that the number of steps needed before it makes any net progress drops from
about sixteen early in training to one later on. When the extra steps come for free, $k=2$ also
lowers FID. Once they are paid for, the result flips: at approximately matched measured wall-clock,
spending the budget on fresh fields gives lower FID than deeper fitting, on both training seeds. The
target itself shows why a fresh field is worth so much. Redrawing the finite support rotates its
direction far more than a parameter update does (cosine $\approx0.3$--$0.6$ against $\approx0.95$),
and a correction that is optimal in field space is not reliably better in FID than a
parameter-free one. For Drifting, fitting each target well and spending compute well are different
goals.
\end{abstract}

\section{Introduction}

Drifting Models~\citep{deng2026drifting} train a one-step generator around a strikingly simple rule.
A vector field $V_{p,q}$, computed from generated samples and real data, says how each generated
sample should move so that the generated distribution $q$ drifts toward the data distribution $p$,
\begin{equation}
    x_{i+1}=x_i+V_{p,q_i}(x_i),
    \label{eq:intended_update_intro}
\end{equation}
and training carries the rule out by freezing the drifted point $x+V$ as a regression target and
taking an optimizer step toward it. Read this way, a training iteration moves samples along the
field.

\begin{figure}[t]
    \centering
    \includegraphics[width=\linewidth]{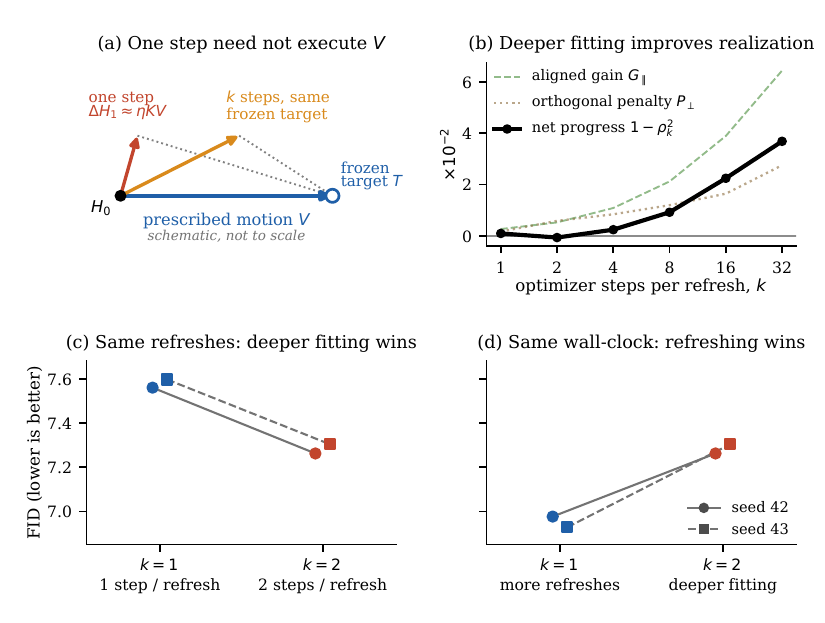}
    \caption{\textbf{A field can become easier to realize and still not be the best place to spend
    compute.} \textbf{(a)}~Drifting prescribes a motion $V$ toward a frozen target; one optimizer step
    executes a tangent-filtered motion $\Delta H_1\approx\eta KV$ that need not coincide with $V$;
    repeated steps against the same frozen target move closer (dotted: remaining residual).
    \textbf{(b)}~Measured on ImageNet $256\times256$ at the state from which all our downstream runs
    start: net progress toward the frozen target grows with depth apart from a small dip at $k=2$.
    \textbf{(c)}~With the number of refreshes fixed, $k=2$ takes two optimizer steps per refresh and
    reaches lower FID on both training seeds. \textbf{(d)}~At approximately matched measured
    wall-clock, $k=1$ affords more refreshes and reaches lower FID on both seeds. Panels (c) and (d)
    share one FID axis.}
    \label{fig:story}
\end{figure}

In practice the move is made in parameter space. The optimizer changes one set of weights shared by
every sample in the batch, so all outputs move together, and to first order a single gradient step
displaces the stacked outputs $\hstack$ by
\begin{equation}
    \Delta\hstack_1=\eta\kmat\vstack,\qquad \kmat=\jmat\jmat^\top,\quad \jmat=\partial\hstack/\partial\thet .
    \label{eq:intro_kv}
\end{equation}
Sample $i$ thus moves by a mixture $\eta\sum_j\kmat_{ij}\vstack_j$ of the motions prescribed for the
whole batch rather than by its own $\vstack_i$, and only in special geometries does $\eta\kmat\vstack$
coincide with $\vstack$. In general the motion the field prescribes and the motion one step executes
differ (\cref{fig:story}a), and \cref{eq:intended_update_intro} is silent about the difference. This paper is about that
difference: our object of study is the training loop that estimates the existing field from data and
realizes it with an optimizer, rather than the definition of the field.

The most direct way to study this gap is to remove the target as a variable. We freeze one target and
take $k$ optimizer steps toward it instead of one; since the target cannot change, whatever happens
to the executed motion is due to fitting alone. We use this frozen-target construction purely as a
probe of realization. Measured on ImageNet~\citep{deng2009imagenet} at $256\times256$ with an exact decomposition of the
frozen-target residual (\cref{sec:geometry}), deeper fitting does close part of the gap
(\cref{fig:story}b), and closes it more readily as training proceeds: the depth at which fitting
first makes net progress toward the target falls from about sixteen steps at checkpoint 5k to a
single step by 20k.

Whether that improvement is worth its cost is a separate question. Every Drifting iteration spends compute on two things:
refreshing the field, by drawing a new finite support of real examples and building a new target,
and realizing it with optimizer steps. Since deeper fitting demonstrably works, a case against it
cannot rest on an optimizer that fails to fit. When the extra steps cost nothing, deeper fitting
wins: at an equal number of refreshes, two steps per refresh reach lower FID than one on both
training seeds (\cref{fig:story}c). When the steps are paid for, the ordering reverses. At
approximately matched measured wall-clock, the run that spends its budget on more refreshes, one step
each, reaches lower FID on both seeds (\cref{fig:story}d), although it used $9.8\%$ less time than its
counterpart on one seed and $1.9\%$ more on the other.

A local, matched-compute version of the same decision points the same way. Starting from a common
parameter state, refreshing the target before the next update beats a second update toward the
stale target in all $40$ probes we ran, across three support sizes and three training stages. The
same probes show where the value of refreshing comes from. At fixed parameters, a newly drawn support
rotates the target direction far more (cosine $0.31$--$0.59$) than one parameter update does at fixed
support ($0.94$--$0.98$), so a second step against an old target keeps fitting one particular sample
of a highly variable quantity, while a refresh replaces it. The size of that rotation is not a
calibrated predictor of the local gain from refreshing, which stays nearly constant across support
sizes that rotate the target by very different amounts (\cref{sec:support}).

The target also depends strongly on how it is constructed. A
same-budget correction to the target, calibrated by a held-out audit that follows a clean
$N^{-1/2}$ law, is not reliably better in FID than a parameter-free split average, and removing a
nominal $10^{-6}$ floor from the production affinity changes the aggregated field by $12\%$.

These measurements describe how the empirical field, its finite support and shared
neural optimization interact inside the Drifting training loop. The executed motion is a
tangent-filtered version of the prescribed one; how much of the prescription is carried out depends
on realization depth, training state and batch geometry; executing a given target better is not the
same as spending compute better; and the target itself is governed by its finite support and by
implementation details. Refresh frequency, realization depth, support construction and affinity
design are separate design axes of Drifting training, which \cref{eq:intended_update_intro} folds
into a single notion of more optimization. Our contributions are:
\begin{enumerate}[leftmargin=*,itemsep=1pt,topsep=2pt]
\item an exact decomposition of frozen-target progress into aligned gain and orthogonal penalty,
which makes the realization gap of Drifting measurable, and the finding on ImageNet that deeper
fitting closes it, increasingly so over training;
\item the refresh--realize reversal: extra fitting lowers FID when it is free, whereas under
approximately matched measured wall-clock refreshing wins on both training seeds and in all $40$
local matched-compute probes;
\item evidence that support refresh, rather than the parameter update, dominates how the empirical
target changes between steps; that field-space optimality of the target does not reliably transfer
to FID; and that a nominal affinity floor materially changes the production field.
\end{enumerate}

\section{Background: Drifting as Refresh and Realize}
\label{sec:background}

Following~\citet{deng2026drifting}, let $x=f_\theta(\epsilon)$ with $\epsilon\sim p_\epsilon$ and
$q_\theta=(f_\theta)_\#p_\epsilon$. Drifting uses positive samples $y^+\sim p$ and negative samples
$y^-\sim q$, a kernel $k(x,y)=\exp(-\norm{x-y}/\tau)$, and the jointly normalized mean-shift
field~\citep{cheng1995mean}
\begin{equation}
    V_{p,q}(x)=\frac{1}{Z_pZ_q}\,\E_{p,q}\!\left[k(x,y^+)\,k(x,y^-)\,(y^+-y^-)\right],
    \label{eq:drift_field}
\end{equation}
which is anti-symmetric, $V_{p,q}=-V_{q,p}$, and vanishes when $p=q$. Training replaces the update
\cref{eq:intended_update_intro} by the frozen-target regression
\begin{equation}
    \mathcal L(\theta)=\E_\epsilon\norm{\phi(f_\theta(\epsilon))
    -\sg\!\big(\phi(f_\theta(\epsilon))+V\big)}_2^2 ,
    \label{eq:original_loss}
\end{equation}
where $\phi$ stacks the weighted features of a pretrained latent MAE~\citep{he2022masked} and the
stop-gradient prevents differentiation through the field~\citep{chen2021exploring}. The ImageNet
implementation uses temperatures $\tau\in\{0.02,0.05,0.2\}$, normalizes each temperature's force to
unit RMS before summing, and trains with AdamW~\citep{loshchilov2019adamw}.

\Cref{eq:drift_field} is an
expectation. The implementation draws a finite positive support $S$ of $N$ examples from a
per-class memory bank, together with a negative support, and forms $\widehat V_S$, which is frozen
into the target of \cref{eq:original_loss}. We call constructing $\widehat V_S$ a
\emph{refresh} and optimizing against the resulting frozen target \emph{realization}. The
original algorithm takes exactly one optimizer step per refresh; we write $k$ for the number of
steps taken per refresh.

\section{Deeper Optimization Does Realize the Frozen Field}
\label{sec:geometry}

\begin{figure}[t]
    \centering
    \includegraphics[width=\linewidth]{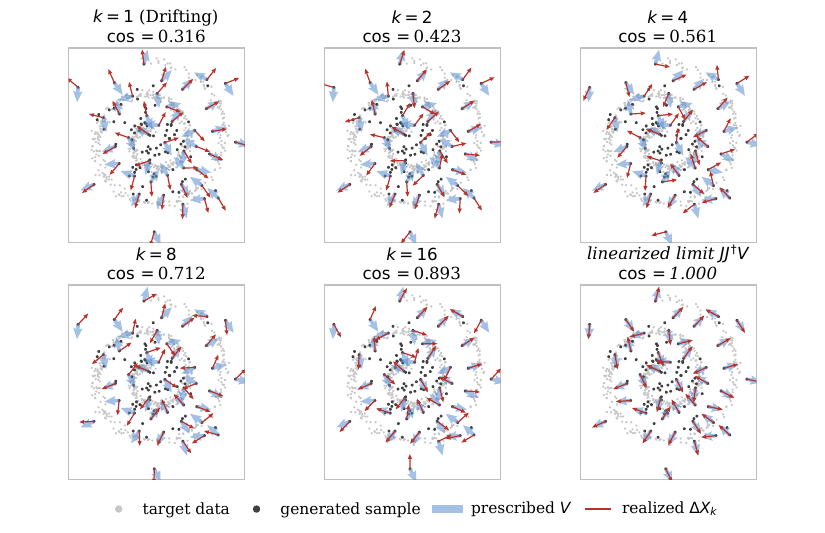}
    \caption{\textbf{Repeated fitting turns the realized motion toward the prescribed one.}
    Swiss-roll Drifting at one frozen training state. Blue: prescribed motion $V$, drawn wide and
    underneath; red: motion realized after $k$ AdamW steps toward the same frozen target. Arrows are
    normalized for display; cosines use the full displacements of all 120 generated samples, and every
    panel shows the same 42 samples. The final panel is not an optimizer run but the $k\to\infty$
    limit $JJ^{\dagger}V$ of the fixed-kernel model at the same state, where $K$ is full rank and the
    prescribed motion is realized exactly.}
    \label{fig:realization}
\end{figure}

\subsection{Why one step need not execute the prescribed motion}

Fix a latent batch and stack the generator features that Drifting regresses into $\hstack(\thet)$.
The loss pulls $\hstack$ toward the frozen target $\hstack+\vstack$, so its gradient in parameter
space is $-\jmat^\top\vstack$. Mapping the resulting parameter change back to feature space gives the
following.

\begin{proposition}[One-step realization]
\label{prop:one_step}
One gradient step with step size $\eta$ moves the features by
$\Delta\hstack_1=\eta\kmat\vstack+O(\norm{\Delta\thet}^2)$, where $\kmat=\jmat\jmat^\top$.
\end{proposition}

Written per sample, $\Delta\hstack_{1,i}=\eta\sum_j\kmat_{ij}\vstack_j$: sample $i$ moves by a
mixture of the motions prescribed for every sample in the batch, weighted by how strongly the shared
parameters couple their outputs. Along the eigendirections of $\kmat$, components of $\vstack$ that
the network can move easily are amplified and the rest are suppressed, so one step generally changes
both the size and the direction of the prescribed motion.

If the target is held fixed and $\kmat$ is
treated as constant, repeated steps correct part of what earlier steps missed.

\begin{proposition}[Fixed-kernel realization model]
\label{thm:k_step}
Under a fixed $\kmat$ and $0<\eta<2/\lambda_{\max}$, $k$ gradient steps against the frozen target
move the features by
\begin{equation}
    \Delta\hstack_k=\big[I-(I-\eta\kmat)^k\big]\vstack ,
    \label{eq:rk}
\end{equation}
which converges to the part of $\vstack$ the network can realize, $\jmat\jmat^{\pinv}\vstack$.
\end{proposition}

This is classical Landweber (Richardson) iteration~\citep{landweber1951}; proofs and spectral
identities are in \cref{app:proofs}. We use it only to motivate the question of whether repeated
fitting realizes the target better. Our ImageNet runs use AdamW with a time-varying preconditioner
and a nonlinear network, so the answer below is measured directly from optimizer-induced
displacements rather than predicted from $\kmat$.

\Cref{fig:realization} shows what repeated fitting looks like in a two-dimensional Swiss-roll version
of Drifting, where the whole state can be drawn. At one training state we freeze the target and let
AdamW take $k$ steps toward it. The realized motion turns toward the prescribed one as $k$ grows, and
its cosine with $V$ rises from $0.32$ at $k=1$ to $0.89$ at $k=16$. The last panel is the
$k\to\infty$ limit of the fixed-kernel model, $\jmat\jmat^{\pinv}\vstack$. At this state $\kmat$ has
full rank ($240$ of $240$ modes), so the linearized model realizes $V$ exactly, and the finite-step
shortfall comes from iteration rather than from a part of $V$ the network cannot reach. The wide
tangent spectrum (condition number $\approx1.9\times10^{3}$) fits this picture, with strong directions
realized within a few steps and weak ones needing many more, although the spectrum alone does not predict the AdamW trajectory. A checkerboard target shows the same finite-$k$ trend under both gradient descent and AdamW
(\cref{tab:toy_optimizer}). These two-dimensional problems make realization visible; every downstream
conclusion below is measured directly on ImageNet under AdamW.

\subsection{A realization diagnostic}

Let $\vstack$ be the prescribed field in the scaled coordinate of the loss, with the training
block weighting, and $\Delta\hstack_k$ the displacement actually produced by $k$ optimizer steps
against the frozen target. Define the executed fraction, relative move and orthogonal component
\begin{equation}
    a_k=\frac{\ip{\Delta\hstack_k}{\vstack}}{\norm{\vstack}^2},\qquad
    m_k=\frac{\norm{\Delta\hstack_k}}{\norm{\vstack}},\qquad
    o_k^2=m_k^2-a_k^2 .
\end{equation}
Because the step-$0$ residual is exactly $-\vstack$, the residual ratio
$\rho_k=\norm{\Delta\hstack_k-\vstack}/\norm{\vstack}$ satisfies
\begin{equation}
    \rho_k^2=(1-a_k)^2+o_k^2
    \qquad\Longleftrightarrow\qquad
    1-\rho_k^2=\underbrace{(2a_k-a_k^2)}_{\text{aligned gain }G_{\parallel,k}}
    -\underbrace{o_k^2}_{\text{orthogonal penalty }P_{\perp,k}} .
    \label{eq:identity}
\end{equation}
The identity holds for any displacement. Its value is as a measurement: it separates motion that
executes the prescribed field from motion that does not, and $\rho_k<1$ exactly when the aligned
gain exceeds the orthogonal penalty. On all measured ImageNet rows it reconstructs the recorded
$\rho_k$ to within $10^{-9}$.

\subsection{Measurement}

At the checkpoint from which every downstream run in \cref{sec:allocation} starts, net progress is
positive at $k=1$, dips slightly at $k=2$, and then grows with depth: from $k=1$ to $k=32$ the
executed fraction rises $24$-fold while the displacement grows only $4$-fold, and net progress
reaches $3.7\times10^{-2}$ (\cref{fig:story}b; exact values in \cref{app:exact_geometry}).
Deeper optimization genuinely improves realization of the frozen field.

\Cref{tab:state} shows that this is a property of the training state. Early in training the
optimizer moves far relative to the prescribed field and mostly orthogonally to it, and net
progress stays negative until $k\approx16$. As training progresses, the one-step
displacement becomes smaller relative to the prescribed field but better aligned with it
($m_1:0.115\rightarrow0.041$, $\cos(\Delta\hstack_1,\vstack):0.024\rightarrow0.033$ from checkpoint
$5$k to $30$k), and the positive-realization threshold moves from roughly $k=16$ to $k=1$. Realization also depends on batch geometry: at checkpoint 5k, increasing the effective batch from $4$ to $64$ moves the first depth with net progress from $k=1$ to $k=16$ (\cref{app:exact_geometry}).

\begin{table}[h!]
\centering\small
\caption{\textbf{The depth at which optimization begins to realize the frozen field falls from
$k\approx16$ to $k=1$ during training.} Net progress $1-\rho_k^2$ ($\times10^{-3}$), $B_{\mathrm{eff}}=64$.
Shaded: first depth with positive net progress.}
\label{tab:state}
\setlength{\tabcolsep}{7pt}
\begin{tabular}{lrrrrrr}
\toprule
checkpoint & $k{=}1$ & $k{=}2$ & $k{=}4$ & $k{=}8$ & $k{=}16$ & $k{=}32$\\
\midrule
5k  & $-7.7$ & $-4.7$ & $-7.1$ & $-2.7$ & \cellcolor{black!10}$+11.0$ & $+31.9$\\
10k & $-1.2$ & $-4.7$ & $-4.6$ & \cellcolor{black!10}$+3.5$ & $+19.5$ & $+37.5$\\
20k & \cellcolor{black!10}$+0.6$ & $-1.2$ & $+3.3$ & $+10.8$ & $+23.3$ & $+41.2$\\
30k & \cellcolor{black!10}$+1.0$ & $-0.6$ & $+2.5$ & $+9.2$ & $+22.5$ & $+36.9$\\
\bottomrule
\end{tabular}
\end{table}

\section{Refresh or Realize? The Reversal}
\label{sec:allocation}

\Cref{sec:geometry} shows that extra optimization against a frozen target realizes it better. Training
poses a different question: under a fixed budget, is that optimization better spent realizing the
current field more accurately or obtaining a fresh one? We compare $k=1$ and $k=2$ optimizer steps per
refresh for latent Drifting on ImageNet $256\times256$ at high positive support, with EMA weights,
CFG $2.0$ and 50k-sample FID, on two independently trained seeds.

\begin{table}[h!]
\centering\small
\caption{\textbf{Extra realization helps when it is free; refreshing wins once it is paid for.}
FID under two budget definitions (lower is better; bold marks the better arm). Wall-clock is the
scheduler-reported elapsed time of each run.}
\label{tab:allocation}
\setlength{\tabcolsep}{6pt}
\begin{tabular}{lcccc}
\toprule
budget & seed & $k=1$ & $k=2$ & better\\
\midrule
\multirow{2}{*}{\shortstack[l]{equal refresh count\\{\footnotesize $k{=}2$ gets $2\times$ updates}}}
 & 42 & 7.56 & \textbf{7.26} & $k=2$\\
 & 43 & 7.60 & \textbf{7.30} & $k=2$\\
\midrule
\multirow{2}{*}{\shortstack[l]{$\approx$ matched measured\\wall-clock}}
 & 42 & \textbf{6.98} {\footnotesize(32.2\,h)} & 7.26 {\footnotesize(35.7\,h)} & $k=1$\\
 & 43 & \textbf{6.93} {\footnotesize(34.5\,h)} & 7.30 {\footnotesize(33.9\,h)} & $k=1$\\
\bottomrule
\end{tabular}
\end{table}

At an equal number of refreshes, $k=2$ receives twice the optimizer updates and reaches lower FID on
both seeds (\cref{tab:allocation}, top). Deeper realization is therefore not a worse procedure, and
the comparison rules out the explanation that $k=2$ simply optimizes badly.

Charging for the extra updates changes the outcome. At approximately matched measured wall-clock the
$k=1$ arm runs $45{,}252$ refreshes and optimizer updates, against $30{,}000$ refreshes and $60{,}000$
updates for $k=2$, and it reaches lower FID on both seeds (\cref{tab:allocation}, bottom):
\begin{quote}
\emph{Under approximately matched measured wall-clock, refreshing the field outperforms deeper
realization of the stale field in both training seeds.}
\end{quote}
The two arms are matched only approximately, and the residual mismatch runs in opposite directions:
on seed~42 the $k=1$ run used $9.8\%$ less time, on seed~43 $1.9\%$ more, while the FID ordering is
the same in both. Run lengths were fixed from an estimated per-step cost before training. The same
allocation decision, isolated at frozen checkpoints with matched compute and no scheduling effects,
favors refreshing in all $40$ probes we ran; \cref{sec:support} presents those probes together with what they reveal about the target.

The operational consequence for this ImageNet regime is direct: when one and two optimizer steps per
refresh are compared at approximately matched measured wall-clock, refreshing more often is the better
allocation than a second step toward the same stale target. This is not an ordering over arbitrary
depths, and we do not claim it for every field construction.

\section{Why Refreshing Pays: Support Dominates Local Target Change}
\label{sec:support}

\subsection{Support refresh rotates the target, and the local decision favors it}

Between consecutive steps the parameters move and the support is redrawn. We separate the two at
frozen checkpoints. From a state $\theta_0$ and target $T_0$ we take one update to $\theta_1$, then
rebuild the target at $\theta_1$ twice: once from the same support ($T_1^{\mathrm{same}}$) and once
from a disjoint, freshly drawn support ($T_1^{\mathrm{new}}$), with labels, generator noise and
guidance draws held fixed. The same probe then makes the allocation decision with matched compute:
\textsc{realize} takes a second update toward the stale $T_0$, \textsc{refresh} takes it toward
$T_1^{\mathrm{new}}$, and both are scored on two held-out supports that neither branch used. We ran
eight probes in each of five cells, covering support sizes $N\in\{64,128,256\}$ at checkpoint $30$k
and training stages $10$k, $20$k and $30$k at $N=256$ (\cref{app:extra_ablations}).

\begin{figure}[t]
    \centering
    \includegraphics[width=\linewidth]{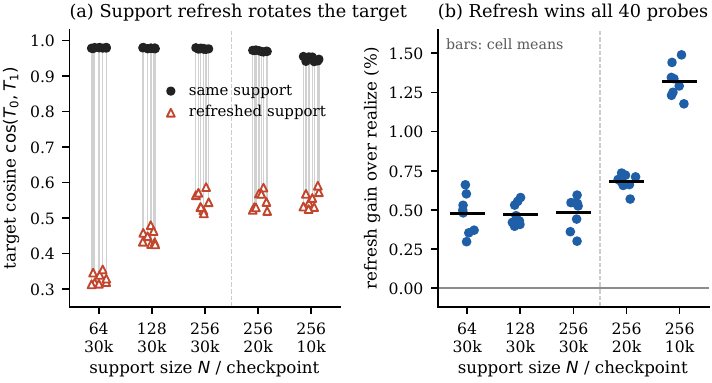}
    \caption{\textbf{Refreshing the support moves the target far more than a parameter update does,
    and refreshing wins every local allocation decision.} Forty probes, eight per position. The first three
    positions vary the support size at checkpoint 30k; the last three vary the training stage at
    $N=256$, sharing the $N=256$, 30k condition.
    \textbf{(a)}~Cosine between the old target and the target rebuilt after one update, from the same
    support (dots) or a freshly drawn one (triangles); lines pair the two for each probe.
    \textbf{(b)}~Held-out loss reduction of \textsc{refresh} over \textsc{realize} at matched compute;
    every probe is positive. Probes within a cell share a model state and are not independent.}
    \label{fig:oracle}
\end{figure}

\Cref{fig:oracle}a shows the first result. Rebuilding the target from the same support after one
update leaves its direction nearly unchanged (cosine $0.94$--$0.98$, mean $0.97$); drawing a new
support rotates it substantially (cosine $0.31$--$0.59$, mean $0.48$), and more so the smaller the
support. In the measured regime, resampling the finite support changes the target direction far more
than the local parameter update itself. In these probes staleness is dominated by support
resampling: a frozen target encodes one draw of a quantity that varies strongly from draw to draw,
while one parameter update barely moves it.

\Cref{fig:oracle}b shows the decision. Refreshing gives the lower held-out loss in all $40$ probes, by
$0.30\%$ to $1.49\%$ (median $0.56\%$). The probes share model states and support pools within a cell,
so we read this as a consistent effect across varied conditions rather than as $40$ independent
trials. The two panels also separate two things that are easy to conflate. At checkpoint $30$k, the
target rotates far more for $N=64$ than for $N=256$, yet the mean refresh gain is almost identical
($0.47\%$ and $0.48\%$); across training stages at $N=256$ the rotation barely changes while the gain
falls from $1.32\%$ at $10$k to $0.48\%$ at $30$k. Support sensitivity explains why refreshing changes
the target; its magnitude is not a calibrated predictor of how much a refresh is worth.

\subsection{Field-space optimality does not reliably transfer to FID}

If the support dominates the target, a natural response is to spend part of each refresh improving
the estimate. We test the cheapest version that keeps the data budget fixed: split the same $N$
positives into disjoint halves, form $V_{\mathrm{split}}=\tfrac12(V_A+V_B)$ and interpolate
$V_\beta=V_N+\beta(V_{\mathrm{split}}-V_N)$, so that $\beta=0$ is Drifting and $\beta=1$ is a
parameter-free split average. A held-out audit against a disjoint large-support reference selects
$\beta^\ast$ without training or FID. In this ImageNet audit $\beta^\ast$ follows an approximately
$N^{-1/2}$ law, with fitted exponent $-0.50$ ($R^2=0.996$) at checkpoint $30$k and $-0.47$ at $60$k
(\cref{fig:calibration}a).

\begin{figure}[t]
    \centering
    \includegraphics[width=\linewidth]{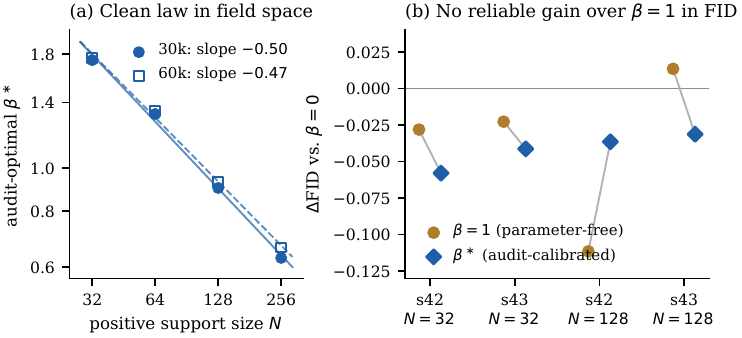}
    \caption{\textbf{A clean law in field space does not become a reliable gain in FID.}
    \textbf{(a)}~The audit-optimal coefficient follows an approximately $N^{-1/2}$ law, stable
    between checkpoints $30$k and $60$k. \textbf{(b)}~Downstream, the calibrated coefficient is not
    reliably better than the parameter-free $\beta=1$: it is better in three cells and worse in one,
    and the mean difference is $-0.005$ FID. FID averaged over three evaluation seeds; exact values in
    \cref{app:calibration}.}
    \label{fig:calibration}
\end{figure}

We froze $\beta^\ast$ before observing any FID and ran deterministic paired continuations on both
seeds. Averaged over three evaluation seeds, the calibrated correction lowers FID in all six
seed--support cells, by $0.040$ on average, but only $13$ of $18$ individual paired evaluations are
favorable, so the effect is comparable to evaluation noise. More to the point, the audit-optimal
coefficient does not reliably outperform the parameter-free split average (\cref{fig:calibration}b):
a correction that is optimal in field space is not reliably optimal for downstream FID. We had also
preregistered that the benefit would shrink as $N$ grows. The single-evaluation-seed results falsified
this; averaging over evaluation seeds restores the predicted ordering, which we report as exploratory.

\subsection{A nominal affinity floor materially changes the production field}

The production field symmetrizes each affinity as
$\sqrt{\max(p^{\mathrm{row}}_{ij}p^{\mathrm{col}}_{ij},\,10^{-6})}$ before forming the attraction and
repulsion coefficients. The floor looks like a numerical safeguard, but at checkpoint $30$k it binds
on about half of the $\tau=0.05$ affinities and on $91\%$ of the $\tau=0.02$ affinities. We recomputed
the field with and without it, holding the checkpoint, labels, generator noise, guidance draws,
generated features, supports and block weights identical, with a correctness gate against the
compiled production code and a materiality threshold fixed before the run (\cref{app:floor}).
Removing only the floor changes the aggregated production field by $12.3\%$ relative $L_2$ (cosine
$0.992$); the $\tau=0.05$ and $\tau=0.02$ branches change by $10.8\%$ and $21.3\%$. Because each
temperature branch is RMS-normalized, every branch keeps its norm, so the floor acts on the direction
of the field rather than its scale. The field Drifting optimizes in this regime includes the floor,
and our conclusions describe that production field.

\section{Related Work}

Drifting~\citep{deng2026drifting} arrived with a strong empirical result and comparatively little
theory, and most follow-up work has been about the field itself. It has been read as a long--short
flow map~\citep{li2026longshort}, as a Wasserstein gradient flow of a kernel-smoothed divergence
\citep{cao2026gradient,gretton2026wasserstein}, and as score matching between kernel-smoothed
distributions~\citep{turan2026secretly,lai2026unified}; other work changes the field, replacing its
displacement direction by kernel gradients~\citep{estebancasadevall2026kernel} or adding momentum to
counter spectral bias~\citep{brown2026secondorder}. These papers analyze the prescribed dynamics. Our
object is the training loop that implements them: how much of the prescribed motion a shared network
carries out, and how compute should be split between rebuilding the field and fitting it. A
better-founded field still has to be estimated from a finite support and realized by an optimizer, so
the two lines of work are complementary.

Drifting is one of several routes to one- or few-step generation, alongside flow matching and
rectified flow~\citep{lipman2023flow,liu2023flow}, consistency models~\citep{song2023consistency} and
shortcut models~\citep{frans2025shortcut}, usually on transformer backbones~\citep{peebles2023dit}.
What sets it apart for our purposes is that its regression target is rebuilt from data at every
step, which is why a refresh--realize trade-off exists at all.

The mechanism of Proposition~\ref{prop:one_step}, a parameter update acting on outputs through
$\kmat=\jmat\jmat^\top$, is the starting point of the neural tangent kernel
literature~\citep{jacot2018ntk,lee2019wide,adlam2020ntk,novak2022fastntk}, which has also been used to
study adversarial generative training~\citep{franceschi2022ganntk}; the frozen-target iteration of
Proposition~\ref{thm:k_step} is Landweber iteration~\citep{landweber1951,engl1996regularization}. We use this
machinery to pose the realization question rather than to answer it, since our measurements are
taken under AdamW, where the fixed-kernel model holds only approximately.

The refresh--realize question also connects to two older ideas. Deep reinforcement learning deliberately
holds regression targets fixed, refreshing target networks only periodically for
stability~\citep{mnih2015human}; Drifting sits at the other end, because its target is a finite-sample
estimate and a stale target mostly encodes the particular support it was built from. And, as with
minibatch estimates of kernel discrepancies such as MMD~\citep{gretton2012kernel}, the empirical
Drifting target depends on the composition of its sample; we measure that dependence at the level of
the direction a training step is asked to fit.

\section{Scope and Conclusion}
\label{sec:limitations}

\paragraph{Scope.} The ImageNet evidence comes from two independently trained source seeds of latent
Drifting at $256\times256$, reported individually. The reversal rests on agreement across evidence
rather than on seed count: both seeds order the arms the same way, the residual wall-clock mismatch
changes sign between them, and the local probes agree. The long-run comparison tests one against two
optimizer steps per refresh at approximately matched measured wall-clock; it does not establish an
ordering over arbitrary depths or under exact compute matching. The support-size law for $\beta^\ast$
is fitted against a finite large-support reference and describes that audit rather than a population
field. The comparisons are made on the production field, affinity floor included, and whether other
affinity constructions shift the balance between refreshing and realizing remains open.

\paragraph{Conclusion.} Drifting prescribes a motion for every generated sample and leaves its
execution to a shared network optimizer. We measured the gap between the two and what it implies for
how training compute should be spent. Fitting a frozen target for longer does move the network closer
to it, more so as training progresses, and with free extra steps it improves FID. Under approximately
matched measured wall-clock the budget is nonetheless better spent recomputing the field, on both
training seeds and in every local probe. Much of the explanation sits in the target: at fixed
parameters, a new finite support changes its direction far more than an optimizer step does, and
calibrating that target in field space does not carry over reliably to FID. Refresh frequency,
realization depth and support construction are therefore design choices of the Drifting training
loop in their own right, and the realization diagnostic introduced here measures the first two
directly.

\section*{Reproducibility Statement}
\Cref{sec:geometry,sec:allocation,sec:support} specify every measurement and comparison.
\Cref{app:implementation} records the configuration, the controls used to compare realization depths
without confounders, both budget definitions, and the deterministic paired protocol behind every
field-correction comparison. \Cref{app:exact_geometry,app:floor,app:calibration} give the exact
values behind every rounded number in the main text, and \cref{app:floor} documents the
production-equivalence gate of the affinity-floor diagnostic.

\section*{AI Use Statement}
In this work, we used generative AI tools (large language models) for the following tasks with
required disclosure: helping develop the conceptual framework (the refresh/realize framing and the
realization decomposition), formulating mathematical claims and assisting in writing proofs,
proposing and refining hypotheses, designing and giving feedback on the experimental methodology
(including controls, paired protocols and diagnostics), implementing methods (training-loop
modifications and diagnostic code), and interpreting results. We have not used generative AI tools
to generate synthetic datasets, and translation, dataset cleaning and reformatting, and qualitative
or thematic data analysis are not applicable to this work. Additionally, we used generative AI tools
for tasks with recommended disclosure: creating and editing software code, suggesting experimental
parameters, creating and modifying scientific figures, brainstorming, identifying and summarizing
related literature, formatting references, suggesting the structure and title of the paper, and
drafting and editing the manuscript. We have reviewed all AI-assisted work. All experiments were run
by the authors on their own compute and every reported number comes from those runs;
AI-assisted code was checked against the production implementation with numerical equivalence
gates and smoke tests before use. We take
responsibility for the final content of this work, including text, claims or artifacts produced
with the aid of generative AI.

\bibliographystyle{iclr2027_conference}
\bibliography{refs}

\clearpage
\appendix
\section{Complete Mathematical Derivations}
\label{app:proofs}

\subsection{Notation and stacking of multiple feature losses}
\label{app:stacking}
The original image objective sums squared losses over feature scales/locations $j$. Suppose a batch contains generated samples indexed by $i=1,\ldots,m$, and each feature block has weight $w_j>0$ after all implementation normalizations. Define
\begin{equation}
 h_\theta(\epsilon_i)
 =\begin{bmatrix}
 \sqrt{w_1}\,\phi_1(f_\theta(\epsilon_i))\\
 \vdots\\
 \sqrt{w_J}\,\phi_J(f_\theta(\epsilon_i))
 \end{bmatrix},
 \qquad
 v_i=\begin{bmatrix}
 \sqrt{w_1}\,V_{i,1}\\ \vdots\\ \sqrt{w_J}\,V_{i,J}
 \end{bmatrix}.
\end{equation}
Then the sum of feature losses is exactly a single Euclidean regression
\begin{equation}
\ell(\theta)=\frac12\left\|\mathbf H(\theta)-\sg(\mathbf H_0+\mathbf V)\right\|^2,
\end{equation}
up to an overall scalar normalization. Any mean reduction over samples/features multiplies the gradient by a constant; throughout the theory this constant is absorbed into the effective step $\eta$. Therefore the results apply to the multi-feature ImageNet objective without requiring each feature block to have equal dimension or scale.

\subsection{
One-step \texorpdfstring{$KV$}{KV} identity with a Taylor remainder
}
\label{app:proof_one_step}

\begin{lemma}[Second-order remainder]
\label{lem:taylor}
Let $H:\R^P\to\R^n$ be continuously differentiable on a convex neighborhood of $\theta_0$, and suppose its Jacobian is $L_J$-Lipschitz in operator norm:
\begin{equation}
\|J(\theta)-J(\theta')\|_{\mathrm{op}}\le L_J\|\theta-\theta'\|.
\end{equation}
Then for any displacement $\delta$ contained in the neighborhood,
\begin{equation}
H(\theta_0+\delta)=H(\theta_0)+J_0\delta+R(\delta),
\qquad
\|R(\delta)\|\le\frac{L_J}{2}\|\delta\|^2.
\label{eq:taylor_bound}
\end{equation}
\end{lemma}

\begin{proof}
By the fundamental theorem of calculus,
\begin{align}
H(\theta_0+\delta)-H(\theta_0)
&=\int_0^1 J(\theta_0+s\delta)\,\delta\,ds\\
&=J_0\delta+\int_0^1\left[J(\theta_0+s\delta)-J_0\right]\delta\,ds.
\end{align}
The second term is $R(\delta)$. Using the Lipschitz assumption,
\begin{align}
\|R(\delta)\|
&\le\int_0^1 L_J\,s\|\delta\|^2ds
=\frac{L_J}{2}\|\delta\|^2.
\end{align}
\end{proof}

\begin{proof}[Proof of Proposition~\ref{prop:one_step}]
The frozen target is $T=H_0+V$ and the loss is
\begin{equation}
\ell(\theta)=\frac12\|H(\theta)-T\|^2.
\end{equation}
At $\theta_0$, $H(\theta_0)-T=-V$, so by the chain rule
\begin{equation}
\nabla_\theta\ell(\theta_0)=J_0^\top(H_0-T)=-J_0^\top V.
\end{equation}
A gradient-descent step gives
\begin{equation}
\delta_1:=\theta_1-\theta_0=-\eta\nabla\ell(\theta_0)=\eta J_0^\top V.
\end{equation}
Applying Lemma~\ref{lem:taylor},
\begin{align}
H(\theta_1)-H_0
&=J_0\delta_1+R(\delta_1)\\
&=\eta J_0J_0^\top V+R(\delta_1)
=\eta KV+R(\delta_1),
\end{align}
with
\begin{equation}
\|R(\delta_1)\|
\le\frac{L_J}{2}\eta^2\|J_0^\top V\|^2.
\end{equation}
\end{proof}

\subsection{Proof and interpretation of tangent-spectral distortion}
\label{app:proof_spectrum}
Let $K=U\Lambda U^\top$ with orthonormal eigenvectors $u_r$ and nonnegative eigenvalues $\lambda_r$. Write $V=\sum_r a_r u_r$. Then
\begin{equation}
KV=\sum_r \lambda_r a_r u_r,
\end{equation}
which gives the band-energy identity
\begin{equation}\label{eq:band_energy}
\ip{\vstack}{\kmat\vstack}=\sum_r\lambda_r a_r^2 .
\end{equation}
For the cosine,
\begin{align}
\ip{V}{KV}&=\sum_r \lambda_r a_r^2,\\
\|V\|^2&=\sum_r a_r^2,\\
\|KV\|^2&=\sum_r\lambda_r^2a_r^2.
\end{align}
Hence
\begin{equation}\label{eq:cos_spectral}
\cos(\vstack,\kmat\vstack)=\frac{\sum_r\lambda_r a_r^2}{\sqrt{\sum_r a_r^2}\sqrt{\sum_r\lambda_r^2a_r^2}} .
\end{equation} Cauchy--Schwarz applied to the sequences $\{|a_r|\}$ and $\{\lambda_r|a_r|\}$ yields $\ip{V}{KV}\le\|V\|\|KV\|$. Equality holds iff the two sequences are linearly dependent on the support of $a_r$, i.e., iff there exists $c$ such that $\lambda_r|a_r|=c|a_r|$ for every active $r$. Thus all eigendirections carrying nonzero field coefficients must share the same eigenvalue.

A useful pairwise view makes the distortion even more explicit. For two active modes $i,j$,
\begin{equation}
\frac{|\ip{u_i}{KV}|/|\ip{u_j}{KV}|}
{|\ip{u_i}{V}|/|\ip{u_j}{V}|}
=\frac{\lambda_i}{\lambda_j}.
\label{eq:pairwise_amp}
\end{equation}
Thus an ill-conditioned tangent spectrum deterministically magnifies the relative prominence of top modes. This statement concerns tangent eigenmodes. It should not be interpreted as a claim about spatial/Fourier image frequencies.

For the finite-$k$ realization, replace $\lambda_r$ by $g_k(\lambda_r)=1-(1-\eta\lambda_r)^k$. Therefore
\begin{equation}
J\delta_k=\sum_r g_k(\lambda_r)a_r u_r.
\end{equation}
For every $\lambda_r>0$, $g_k(\lambda_r)\to1$ as $k\to\infty$ under the stability condition. Hence the relative reweighting between any two positive-eigenvalue modes disappears asymptotically:
\begin{equation}
\frac{g_k(\lambda_i)}{g_k(\lambda_j)}\to1.
\end{equation}
This is the precise sense in which repeated frozen-target fitting flattens the one-step tangent filter.

\subsection{
Exact linearized \texorpdfstring{$k$}{k}-step dynamics
}
\label{app:proof_kstep}

Freeze $J=J_0$ and $V$ and optimize the local least-squares model
\begin{equation}
F(\delta)=\frac12\|J\delta-V\|^2,
\qquad
\nabla F(\delta)=J^\top(J\delta-V).
\end{equation}
Starting from $\delta_0=0$, gradient descent is
\begin{equation}
\delta_{s+1}=\delta_s+\eta J^\top(V-J\delta_s).
\label{eq:gd_appendix}
\end{equation}
Define output residual $r_s=V-J\delta_s$. Then
\begin{align}
r_{s+1}
&=V-J\delta_{s+1}\\
&=V-J\delta_s-\eta JJ^\top(V-J\delta_s)\\
&=(I-\eta K)r_s.
\end{align}
Since $r_0=V$, induction yields
\begin{equation}
r_k=(I-\eta K)^kV.
\end{equation}
Because $J\delta_k=V-r_k$,
\begin{equation}
J\delta_k=\left[I-(I-\eta K)^k\right]V,
\end{equation}
proving Proposition~\ref{thm:k_step}.

An explicit expression for the parameter iterate follows by summing \cref{eq:gd_appendix}:
\begin{align}
\delta_k
&=\eta\sum_{s=0}^{k-1}J^\top r_s\\
&=\eta J^\top\sum_{s=0}^{k-1}(I-\eta K)^sV.
\label{eq:delta_geometric}
\end{align}
This expression is valid even if $K$ is singular; no inverse has been introduced.

\subsection{SVD proof of the Moore--Penrose limit}
\label{app:proof_pinv}
Let the compact singular value decomposition be
\begin{equation}
J=U_r\Sigma_rW_r^\top,
\qquad
\Sigma_r=\diag(\sigma_1,\ldots,\sigma_r),
\end{equation}
where $\sigma_i>0$ and $r=\mathrm{rank}(J)$. Extend $U_r$ to an orthonormal basis and decompose
\begin{equation}
V=\sum_{i=1}^r a_i u_i+V_\perp,
\qquad
V_\perp\in\Null(J^\top).
\end{equation}
Because $K=U_r\Sigma_r^2U_r^\top$, the finite-$k$ filter
\begin{equation}\label{eq:k_filter}
g_k(\lambda)=1-(1-\eta\lambda)^k
\end{equation}
gives
\begin{equation}
J\delta_k
=\sum_{i=1}^r
\left[1-(1-\eta\sigma_i^2)^k\right]a_i u_i.
\label{eq:output_svd}
\end{equation}
The null component $V_\perp$ is never realized.

For the parameter iterate, use $J^\top u_i=\sigma_iw_i$ in \cref{eq:delta_geometric}:
\begin{align}
\delta_k
&=\sum_{i=1}^r
\eta\sigma_i a_i
\sum_{s=0}^{k-1}(1-\eta\sigma_i^2)^s w_i\\
&=\sum_{i=1}^r
\frac{a_i}{\sigma_i}
\left[1-(1-\eta\sigma_i^2)^k\right]w_i.
\label{eq:param_svd}
\end{align}
If $0<\eta<2/\lambda_{\max}(K)=2/\sigma_1^2$, then $|1-\eta\sigma_i^2|<1$ for every nonzero singular direction. Hence
\begin{equation}
\delta_k\to\sum_{i=1}^r\frac{a_i}{\sigma_i}w_i
=W_r\Sigma_r^{-1}U_r^\top V
=J^\dagger V.
\end{equation}
Similarly, \cref{eq:output_svd} converges to
\begin{equation}
\sum_{i=1}^r a_i u_i=U_rU_r^\top V=JJ^\dagger V.
\end{equation}
Because every iterate $\delta_k$ lies in $\Range(J^\top)$, no component in $\Null(J)$ is introduced. The limit is therefore the minimum-Euclidean-norm least-squares solution.

\subsection{
Fixed-point interpretation and implicit solution of
\texorpdfstring{$K^{-1}V$}{K inverse V}
}
\label{app:fixedpoint}
There is an equivalent output-space interpretation. Since $\delta_0=0$, every iterate can be written $\delta_k=J^\top u_k$ for some $u_k$. Substituting into \cref{eq:gd_appendix},
\begin{equation}
J^\top u_{k+1}
=J^\top\left[u_k+\eta(V-Ku_k)\right].
\end{equation}
One admissible recursion is
\begin{equation}
u_{k+1}=u_k+\eta(V-Ku_k),\qquad u_0=0.
\label{eq:richardson}
\end{equation}
This is Richardson/Landweber fixed-point iteration for $Ku=V$. If $K$ is invertible, its fixed point is $u^*=K^{-1}V$, yielding $\delta^*=J^\top K^{-1}V=J^\dagger V$. If $K$ is singular, the iteration converges on the positive-eigenvalue subspace and the induced parameter limit remains $J^\dagger V$.

This identity formalizes the central algorithmic observation: the network parameters implicitly carry $J^\top u_k$, so an explicit output-space inverse variable need not be stored.

\subsection{Finite-step spectral regularization}
\label{app:regularization}
For a $K$-eigenvector with eigenvalue $\lambda$, the output gain is $g_k(\lambda)=1-(1-\eta\lambda)^k$. If $0<\eta\le1/\lambda_{\max}$, then $0\le1-\eta\lambda\le1$ and
\begin{equation}
\frac{\partial}{\partial k}g_k(\lambda)>0,
\qquad
\frac{\partial}{\partial\lambda}g_k(\lambda)>0
\end{equation}
in the continuous extension. Thus large-$\lambda$ directions are realized first. For $\eta\lambda\ll1$, the binomial expansion gives
\begin{equation}
g_k(\lambda)=k\eta\lambda+O(k^2\eta^2\lambda^2).
\end{equation}
The parameter coefficient from \cref{eq:param_svd} is
\begin{equation}
c_{k,i}=\frac{a_i}{\sigma_i}\left[1-(1-\eta\sigma_i^2)^k\right].
\end{equation}
For weak singular values, $\eta\sigma_i^2\ll1$,
\begin{equation}
c_{k,i}\approx k\eta\sigma_i a_i,
\end{equation}
whereas the exact pseudoinverse coefficient is $a_i/\sigma_i$. Early stopping therefore avoids the $1/\sigma_i$ amplification of extremely weak directions. This is the precise sense in which finite $k$ acts as iterative regularization.

For a desired directional realization error $|1-g_k(\lambda)|\le\varepsilon$,
\begin{equation}
|1-\eta\lambda|^k\le\varepsilon
\quad\Longleftrightarrow\quad
k\ge\frac{\log\varepsilon}{\log|1-\eta\lambda|}.
\end{equation}
When $\eta\lambda\ll1$, $\log(1-\eta\lambda)\approx-\eta\lambda$, giving
\begin{equation}\label{eq:k_needed}
k\gtrsim\frac{\log(1/\varepsilon)}{\eta\lambda}.
\end{equation}
Small $k$ is therefore a low-order iterative correction rather than an accurate pseudoinverse solve, which is consistent with the measured residuals in \cref{tab:geometry_exact}.

\subsection{Nonlinear network: what remains exact}
\label{app:nonlinear}
The linearized formulas are exact for the frozen-J model. For the real nonlinear network, the optimization objective
\begin{equation}
\min_\theta\frac12\|H(\theta)-T\|^2
\end{equation}
is itself exact; only its interpretation as a fixed polynomial in $K_0$ is local. At micro-step $s$,
\begin{equation}
\theta_{s+1}=\theta_s-\eta J(\theta_s)^\top(H(\theta_s)-T).
\end{equation}
If the Jacobian remains close to $J_0$ over the micro-trajectory, Lemma~\ref{lem:taylor} controls the mismatch between the nonlinear feature displacement and $J_0\delta_s$. Accordingly, we use the local model to interpret realization rather than to predict it: every realization measurement reported in this paper uses the actual finite feature displacement produced by the optimizer, not a linearized prediction.

\section{Preconditioned Realization and AdamW}
\label{app:preconditioned}
The public image system uses AdamW rather than Euclidean gradient descent. A useful intermediate model freezes a symmetric positive-definite parameter preconditioner $P\succ0$ during one outer realization solve:
\begin{equation}
\delta_{s+1}=\delta_s+\eta P J^\top(V-J\delta_s).
\label{eq:precond_iter}
\end{equation}
Define
\begin{equation}
K_P=JPJ^\top.
\end{equation}
Then the residual obeys
\begin{equation}
r_{s+1}=(I-\eta K_P)r_s,
\end{equation}
so
\begin{equation}
J\delta_k=[I-(I-\eta K_P)^k]V.
\end{equation}
The derivation is identical to that of Proposition~\ref{thm:k_step}. Let $A=JP^{1/2}$ and write $\delta=P^{1/2}w$. The fixed-point problem becomes $Aw\approx V$, whose minimum-Euclidean-norm solution is $w^*=A^\dagger V$. Therefore
\begin{equation}
\delta^*=P^{1/2}A^\dagger V
=PJ^\top(JPJ^\top)^\dagger V.
\label{eq:weighted_pinv}
\end{equation}
This is the solution minimizing $\delta^\top P^{-1}\delta$ among least-squares solutions. A fixed diagonal preconditioner can therefore flatten or reshape the relevant tangent spectrum and accelerate some directions.

Actual AdamW has time-varying first- and second-moment states and decoupled weight decay, so \cref{eq:precond_iter} is not an exact model of the full optimizer. We use it only to motivate two diagnostics: (i) compare $J(-g)$ with the actual $J\Delta\theta_{\rm AdamW}$ at the same state; and (ii) compare repeated AdamW to repeated SGD on the same frozen target. No theorem in the paper should claim that AdamW converges to the Euclidean Moore--Penrose solution.

\section{Controlled Two-Dimensional Study}
\label{app:toy}

The two-dimensional study makes realization and allocation visible. The target is a Swiss roll,
the generator is an MLP acting on a fixed latent bank, and the drift field is built by the same
attraction--repulsion construction as the ImageNet implementation (single temperature, generated
samples as the repulsive support). At a fixed intermediate training state we freeze the target
$T=X_0+V$, clone the model, and take $k\in\{1,2,4,8,16\}$ micro-steps with either gradient descent
or AdamW, measuring the cosine between the realized displacement $\Delta X_k=X_k-X_0$ and $V$ and
the residual $\norm{V-\Delta X_k}/\norm{V}$ (\cref{tab:toy_optimizer}, \cref{fig:realization}).
We then train from scratch for $10{,}000$ outer iterations at each $k$ and compare the
sliced-Wasserstein distance to the target under four budget definitions
(\cref{fig:toy_budgets}).

\begin{table}[t]
\centering
\scriptsize
\caption{
\textbf{Finite-$k$ realization under GD and AdamW.}
Cosine measures directional agreement between the prescribed field $V$ and the actual finite output displacement $\Delta X_k$;
residual is $\|V-\Delta X_k\|/\|V\|$.
The audit is performed at a fixed intermediate training state with the same frozen target for all $k$.
The GD learning rate is chosen from the local stability condition, whereas AdamW uses the practical fixed learning rate; the table is therefore intended to compare the \emph{trend with $k$}, not the absolute performance of the two optimizers.
}
\label{tab:toy_optimizer}
\resizebox{\linewidth}{!}{%
\begin{tabular}{llcccccc}
\toprule
Dataset & Optimizer & Metric
& $k=1$ & $k=2$ & $k=4$ & $k=8$ & $k=16$ \\
\midrule

\multirow{4}{*}{Swiss roll}
& \multirow{2}{*}{GD}
& cosine
& 0.342 & 0.448 & 0.516 & 0.585 & 0.651 \\
&
& residual
& 0.953 & 0.926 & 0.889 & 0.841 & 0.785 \\

\cmidrule(lr){2-8}

&
\multirow{2}{*}{AdamW}
& cosine
& 0.316 & 0.423 & 0.561 & 0.712 & \textbf{0.893} \\
&
& residual
& 0.949 & 0.909 & 0.834 & 0.703 & \textbf{0.457} \\

\midrule

\multirow{4}{*}{Checkerboard}
& \multirow{2}{*}{GD}
& cosine
& 0.370 & 0.469 & 0.553 & 0.621 & 0.694 \\
&
& residual
& 0.960 & 0.933 & 0.892 & 0.836 & 0.765 \\

\cmidrule(lr){2-8}

&
\multirow{2}{*}{AdamW}
& cosine
& 0.277 & 0.449 & 0.615 & 0.775 & \textbf{0.927} \\
&
& residual
& 0.961 & 0.894 & 0.792 & 0.637 & \textbf{0.374} \\

\bottomrule
\end{tabular}}
\end{table}

\begin{figure}[t]
    \centering
    \includegraphics[width=\linewidth]{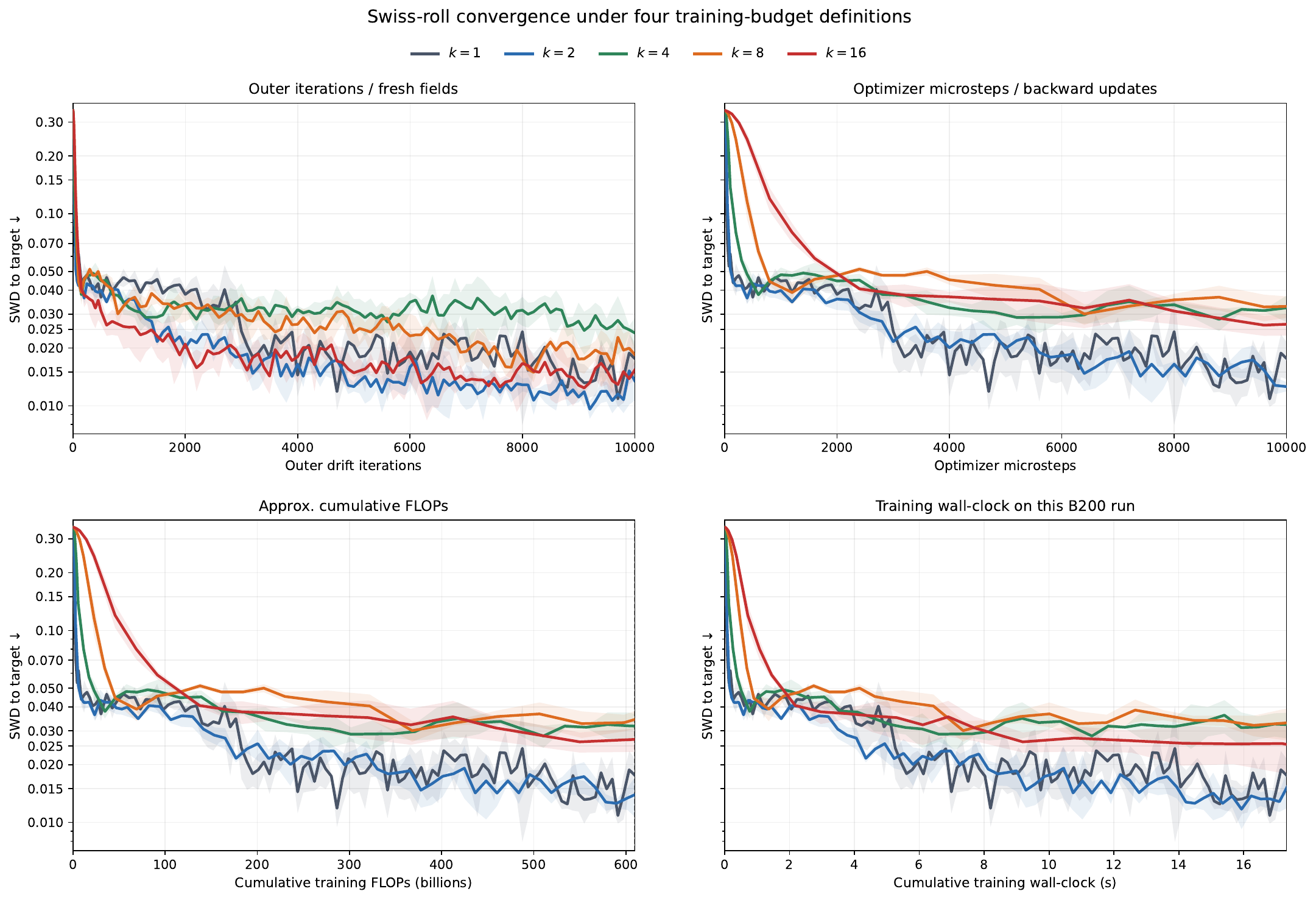}
    \caption{\textbf{How realization depth ranks depends on what is held fixed.} Swiss-roll
    sliced-Wasserstein distance under four budget definitions. Ranking the same runs by outer
    iterations, optimizer micro-steps, FLOPs and wall-clock gives different orderings, which is
    why the ImageNet comparison reports measured elapsed time for every arm.}
    \label{fig:toy_budgets}
\end{figure}

\subsection{Long-horizon toy training under matched budgets}
\label{app:toy_budget}
\Cref{fig:toy_budgets} reports the full training trajectories behind
\cref{tab:toy_optimizer}. Each curve is the mean over three paired seeds with one standard deviation
shaded; the vertical axis is logarithmic. For the optimizer-update, FLOP and wall-clock comparisons,
curves are shown only over the budget range available to every realization depth.

\section{ImageNet Protocol}
\label{app:implementation}

\subsection{Base configuration and reproduction}
All ImageNet results use the public JAX release of~\citet{deng2026drifting} in its latent
$256\times256$ configuration with a pretrained latent MAE feature space, temperatures
$\{0.02,0.05,0.2\}$, AdamW, and EMA weights for evaluation. We reproduced the public ablation
baseline before running any comparison in this paper; all comparisons branch from reproduced
checkpoints rather than from released weights. FID uses 50k generated samples at CFG $2.0$ with a
fixed evaluation batch size, and every arm of a comparison uses identical evaluation arguments so
that generation noise is common across arms.

\subsection{Realization depth without hidden confounders}
For $k>1$ the frozen target must be recomputed from the parameters at the start of the outer
iteration, never from the parameters as they move within it; our implementation rematerializes the
target from the captured outer parameters for every micro-step after the first. Exponential moving
average is updated once per outer iteration for every $k$, so arms compared at a fixed outer
horizon receive the same number of EMA updates. Weight decay is applied per micro-step at the rate
that makes the per-outer decay identical across $k$, i.e.\ $(1-\eta\lambda_{\mathrm{inner}})^k
=1-\eta\lambda_{\mathrm{outer}}$. Without these three controls a $k$ comparison silently changes
target staleness, averaging horizon and regularization at the same time.

\subsection{Resource-matched protocol}
Two budget definitions appear in \cref{sec:allocation}. \emph{Equal outer iterations} fixes the
number of refreshes and lets $k$ determine the number of optimizer updates; this favors larger
$k$, and we report it first. \emph{Approximately matched measured wall-clock} fixes the scheduler
elapsed time; run lengths were set from an estimated per-step cost before the runs and the
realized elapsed times are reported as measured, not as targets. We report both, together with
optimizer-update counts, because ranking by outer iterations, by micro-steps, by FLOPs and by
wall-clock can disagree (\cref{fig:toy_budgets}).

\subsection{Deterministic paired protocol for field-correction comparisons}
Every $\beta$ comparison in \cref{sec:support} is run as a matched pair: the two arms restore the
same full training state, including optimizer moments and EMA, and are seeded per step before the
data iterator and the memory-bank draw, so the two runs see identical batches, identical
generator noise and identical support draws. The only difference between the arms is $\beta$.
Each arm's positive support is drawn from a per-class bank whose size is twice the support size,
and bank occupancy is asserted to exceed the support size at the first scientific update so that
sampling never silently falls back to sampling with replacement.

\section{Exact Realization Measurements}
\label{app:diagnostics}
\label{app:exact_geometry}

All quantities below are computed from optimizer-induced displacements at frozen checkpoints in the
scaled coordinate of the loss with the training block weighting. The identity
$\rho_k^2=(1-a_k)^2+o_k^2$ reconstructs every recorded $\rho_k$ to within $1.4\times10^{-9}$.

\begin{table}[h]
\centering\small
\caption{Full decomposition at checkpoint $30$k, $B_{\mathrm{eff}}=64$ (the data behind
\cref{fig:story}b).}
\label{tab:geometry_exact}
\begin{tabular}{rrrrrrrr}
\toprule
$k$ & $\rho_k$ & $\cos$ & $m_k$ & $a_k$ & $G_{\parallel,k}$ & $P_{\perp,k}$ & $1-\rho_k^2$\\
\midrule
1  & 0.999505 & 0.03266 & 0.04146 & 0.001354 & 0.002706 & 0.001717 & $+0.000989$\\
2  & 1.000303 & 0.03446 & 0.07681 & 0.002647 & 0.005287 & 0.005893 & $-0.000606$\\
4  & 0.998774 & 0.05935 & 0.09208 & 0.005465 & 0.010900 & 0.008449 & $+0.002451$\\
8  & 0.995368 & 0.09702 & 0.11004 & 0.010676 & 0.021238 & 0.011995 & $+0.009243$\\
16 & 0.988698 & 0.15148 & 0.12983 & 0.019666 & 0.038945 & 0.016468 & $+0.022477$\\
32 & 0.981388 & 0.19359 & 0.16914 & 0.032744 & 0.064415 & 0.027537 & $+0.036878$\\
\bottomrule
\end{tabular}
\end{table}

\paragraph{Batch geometry.} Realizability also depends on the effective batch. At checkpoint $5$k,
the one-step residual ratio is $\rho_1=0.99961$, $0.99964$, $1.00189$, $1.00365$ and $1.00385$ for
$B_{\mathrm{eff}}=4,8,16,32,64$, and the first depth with $\rho_k<1$ is $k=1,1,4,8,16$ respectively.
Larger batches spread a fixed parameter update over more generated samples, so a single step executes
less of each sample's prescribed motion.

\section{Support Sensitivity and Local Allocation Probes}
\label{app:extra_ablations}

\paragraph{Support sensitivity.} Both cosines in \cref{sec:support} are measured at frozen
checkpoints with labels, generator noise and guidance draws held fixed. The same-support value takes
one optimizer step and rebuilds the target from the identical support; the resampled value holds the
parameters fixed and redraws only the support. Measuring both from the same probe ensures that the
difference reflects the target, not the probe.

\paragraph{Local refresh/realize probes.} Each probe starts from a checkpoint $\theta_0$ and a
target $T_0$ built from support $S_0$, takes one optimizer update to $\theta_1$, and rebuilds the
target at $\theta_1$ from the same support ($T_1^{\mathrm{same}}$) and from a disjoint freshly drawn
support ($T_1^{\mathrm{new}}$), with labels, generator noise and guidance draws fixed. The two cosines
in \cref{fig:oracle}a are $\cos(T_0,T_1^{\mathrm{same}})$, which isolates the parameter-induced change,
and $\cos(T_0,T_1^{\mathrm{new}})$, which adds support resampling. From the same $\theta_1$,
\textsc{realize} takes a second update toward $T_0$ and \textsc{refresh} a second update toward
$T_1^{\mathrm{new}}$; both are evaluated on two held-out supports unseen by either branch, giving
held-out losses $L_R$ and $L_F$. \Cref{fig:oracle}b plots $100\,(L_R-L_F)/L_R$, positive when
refreshing is better. Each cell holds eight probes.

\begin{table}[h]
\centering\small
\caption{Local refresh/realize probes, cell means over eight probes each.}
\label{tab:oracle_cells}
\begin{tabular}{lcccc}
\toprule
cell & $\cos(T_0,T_1^{\mathrm{same}})$ & $\cos(T_0,T_1^{\mathrm{new}})$ & refresh gain & refresh wins\\
\midrule
$N=64$, 30k  & 0.9790 & 0.3300 & 0.474\% & 8/8\\
$N=128$, 30k & 0.9780 & 0.4463 & 0.472\% & 8/8\\
$N=256$, 30k & 0.9767 & 0.5456 & 0.482\% & 8/8\\
$N=256$, 20k & 0.9704 & 0.5466 & 0.678\% & 8/8\\
$N=256$, 10k & 0.9476 & 0.5526 & 1.321\% & 8/8\\
\midrule
all 40 & 0.9704 & 0.4842 & 0.685\% (median 0.562\%) & 40/40\\
\bottomrule
\end{tabular}
\end{table}

One probe ($N=256$, $10$k, probe~7) was initially rejected by the $10^{-4}$ production-equivalence
gate on the target decomposition (relative error $1.13\times10^{-4}$) before any oracle quantity was
computed. It was rerun at the $5\times10^{-4}$ tolerance adopted from the three-way numerical audit
described in \cref{app:floor}, with the checkpoint, support draws, updates, held-out evaluation and
decision rule unchanged. The tolerance was chosen from that audit, not from any probe outcome.

\section{Affinity-Floor Diagnostic}
\label{app:floor}

\paragraph{Protocol.} At checkpoint $30$k with temperatures $(0.2,0.05,0.02)$, $64$ positives, $16$
negatives and micro-batch $4$, we evaluate four micro-slices. For each slice the clip-on and clip-off
fields share the checkpoint, parameters, labels, generator noise, guidance draws, generated features,
positive and negative supports, feature blocks and block weights; the only change is
$\sqrt{\max(p^{\mathrm{row}}p^{\mathrm{col}},10^{-6})}$ versus $\sqrt{p^{\mathrm{row}}p^{\mathrm{col}}}$.
The field construction reproduces the production code: weighted distance scale including the
unconditional-negative guidance weight, generated-sample self mask, row and column softmaxes,
attraction and repulsion coefficients, and per-temperature RMS normalization. Metrics use the
training block weighting and are pooled from summed inner products and norms across slices rather
than by averaging per-slice cosines. The materiality criterion (aggregated cosine $<0.99$ or
relative $L_2>0.05$) was fixed before the run.

\paragraph{Correctness.} Before any metric is reported, each block's clip-on target is compared with
the compiled production \texttt{drift\_target}. Our mirror and an independent established audit
mirror agree to $2.9\times10^{-6}$, while both differ from the compiled production function by
$\approx2.3\times10^{-4}$ on the worst block, a numerical discrepancy shared by two independently
written mirrors. The equivalence gate was therefore set at $5\times10^{-4}$ on the basis of this
three-way comparison, independently of the scientific outcome. All four slices passed, with maximum
relative target discrepancies $6.4\times10^{-5}$, $9.8\times10^{-5}$, $1.6\times10^{-5}$ and
$1.4\times10^{-4}$, and scale discrepancies below $2.3\times10^{-7}$.

\begin{table}[h]
\centering\small
\caption{Affinity-floor diagnostic per slice and pooled. Relative change is
$\norm{V_{\mathrm{clip}}-V_{\mathrm{noclip}}}/\norm{V_{\mathrm{clip}}}$; norm ratio is
$\norm{V_{\mathrm{noclip}}}/\norm{V_{\mathrm{clip}}}$. Every temperature branch keeps unit norm
ratio because it is RMS-normalized; the change is in direction.}
\label{tab:floor_exact}
\begin{tabular}{llcccc}
\toprule
slice & field & floor active & cosine & rel.\ change & norm ratio\\
\midrule
0 & $\tau=0.2$  & 1.27\% & 0.99976 & 0.022 & 1.000\\
  & $\tau=0.05$ & 49.83\% & 0.99121 & 0.133 & 1.000\\
  & $\tau=0.02$ & 90.42\% & 0.96938 & 0.247 & 1.000\\
  & aggregated & --- & 0.98930 & 0.147 & 0.975\\
\addlinespace[2pt]
1 & $\tau=0.2$  & 1.28\% & 0.99968 & 0.025 & 1.000\\
  & $\tau=0.05$ & 48.49\% & 0.99583 & 0.091 & 1.000\\
  & $\tau=0.02$ & 90.21\% & 0.98267 & 0.186 & 1.000\\
  & aggregated & --- & 0.99441 & 0.106 & 0.985\\
\addlinespace[2pt]
2 & $\tau=0.2$  & 1.15\% & 0.99941 & 0.034 & 1.000\\
  & $\tau=0.05$ & 47.49\% & 0.99482 & 0.102 & 1.000\\
  & $\tau=0.02$ & 90.12\% & 0.98060 & 0.197 & 1.000\\
  & aggregated & --- & 0.99329 & 0.116 & 0.984\\
\addlinespace[2pt]
3 & $\tau=0.2$  & 1.45\% & 0.99976 & 0.022 & 1.000\\
  & $\tau=0.05$ & 53.95\% & 0.99478 & 0.102 & 1.000\\
  & $\tau=0.02$ & 91.89\% & 0.97647 & 0.217 & 1.000\\
  & aggregated & --- & 0.99271 & 0.121 & 0.980\\
\midrule
pooled & $\tau=0.2$  & \phantom{0}1.29\% & 0.99965 & 0.026 & 1.000\\
       & $\tau=0.05$ & 49.94\% & 0.99416 & 0.108 & 1.000\\
       & $\tau=0.02$ & 90.66\% & 0.97728 & 0.213 & 1.000\\
       & \textbf{aggregated} & --- & \textbf{0.99246} & \textbf{0.123} & 0.981\\
       & physical $\sigma V$ & --- & 0.99550 & 0.095 & 0.985\\
\bottomrule
\end{tabular}
\end{table}

\section{Field-Correction Calibration}
\label{app:calibration}

\paragraph{Audit.} For each $N$, the same $N$ positives are split into disjoint halves and the
least-squares $\beta^\ast$ is fitted against a disjoint large-support reference in field space,
with $16$ probes at checkpoint $30$k and $8$ at $60$k. Coefficients at $N=32,64,128,256$ are
$1.7409$, $1.3192$, $0.9015$, $0.6288$ at $30$k (exponent $-0.4957$, $R^2=0.9958$) and $1.7594$,
$1.3410$, $0.9305$, $0.6634$ at $60$k (exponent $-0.4749$, $R^2=0.9966$). Probe-level spread is a
stability diagnostic, not a confidence interval, since probes share parameters and support pools.

\paragraph{Downstream.} Continuations run from each seed's checkpoint $30$k to $35$k under a
deterministic paired protocol in which the two arms see identical batches, generator noise and
support draws, with the positive bank twice the support size. \Cref{tab:calibration_exact} lists
every paired difference; \cref{tab:beta_one} compares the calibrated coefficient with $\beta=1$.

\begin{table}[h]
\centering\small
\caption{Paired $\mathrm{FID}(\beta^\ast)-\mathrm{FID}(0)$ for each evaluation seed.}
\label{tab:calibration_exact}
\begin{tabular}{llrrrr}
\toprule
seed & $N$ & eval 0 & eval 1 & eval 2 & mean\\
\midrule
42 & 32  & $-0.0941$ & $-0.0341$ & $-0.0455$ & $-0.0579$\\
42 & 64  & $-0.0943$ & $-0.0856$ & $+0.0524$ & $-0.0425$\\
42 & 128 & $-0.1405$ & $+0.1000$ & $-0.0690$ & $-0.0365$\\
43 & 32  & $-0.0484$ & $+0.0143$ & $-0.0899$ & $-0.0413$\\
43 & 64  & $-0.0500$ & $-0.0489$ & $+0.0026$ & $-0.0321$\\
43 & 128 & $-0.0732$ & $+0.0154$ & $-0.0362$ & $-0.0313$\\
\bottomrule
\end{tabular}
\end{table}

\begin{table}[h]
\centering\small
\caption{Calibrated $\beta^\ast$ versus the parameter-free $\beta=1$ (FID averaged over three
evaluation seeds). The mean difference is $-0.0046$ and the sign disagrees across seeds at $N=128$.}
\label{tab:beta_one}
\begin{tabular}{llrrrr}
\toprule
seed & $N$ & $\beta=0$ & $\beta=1$ & $\beta^\ast$ & $\beta^\ast-\beta{=}1$\\
\midrule
42 & 32  & 8.0228 & 7.9947 & 7.9649 & $-0.0298$\\
43 & 32  & 8.0661 & 8.0434 & 8.0248 & $-0.0186$\\
42 & 128 & 7.5843 & 7.4728 & 7.5478 & $+0.0750$\\
43 & 128 & 7.6194 & 7.6329 & 7.5881 & $-0.0448$\\
\bottomrule
\end{tabular}
\end{table}

\end{document}